\documentclass[11pt]{article}

\usepackage[final]{acl}

\usepackage{times}
\usepackage{latexsym}
\usepackage{algorithm}
\usepackage{algorithmic}
\usepackage{amssymb} 
\usepackage{amsmath}
\usepackage{booktabs}
\usepackage{multirow}
\usepackage{amsthm}
\newtheorem{theorem}{Theorem}
\usepackage[T1]{fontenc}

\usepackage[utf8]{inputenc}

\usepackage{microtype}

\usepackage{inconsolata}

\usepackage{graphicx}

\title{ELMM: Efficient Lightweight Multimodal Large Language Models for Multimodal Knowledge Graph Completion}

\author{
 \textbf{Wei Huang\textsuperscript{1}},
 \textbf{Peining Li\textsuperscript{1}},
 \textbf{Meiyu Liang\textsuperscript{1}},
 \textbf{Xu Hou\textsuperscript{1}},
 \textbf{Junping DU\textsuperscript{1}},
 \textbf{Yingxia Shao\textsuperscript{1}},
 \textbf{Guanhua Ye\textsuperscript{1}},
 \\
 \textbf{Wu Liu\textsuperscript{1}},
 \textbf{KangKang Lu\textsuperscript{1}},
 \textbf{Yang Yu\textsuperscript{1}}
\\
 \textsuperscript{1}School of Computer Science, Beijing University of Posts and Telecommunications, \\
}

\begin{document}
\maketitle
\begin{abstract}
Multimodal Knowledge Graphs (MKGs) extend traditional knowledge graphs by incorporating visual and textual modalities, enabling richer and more expressive entity representations. However, existing MKGs often suffer from incompleteness, which hinders their effectiveness in downstream tasks. Therefore, multimodal knowledge graph completion (MKGC) task is receiving increasing attention. While large language models (LLMs) have shown promise for knowledge graph completion (KGC), their application to the multimodal setting remains underexplored. Moreover, applying Multimodal Large Language Models (MLLMs) to the task of MKGC introduces significant challenges: (1) the large number of image tokens per entity leads to semantic noise and modality conflicts, and (2) the high computational cost of processing large token inputs. To address these issues, we propose \textbf{E}fficient \textbf{L}ightweight \textbf{M}ultimodal Large Language \textbf{M}odels (\textbf{ELMM}) for MKGC. ELMM proposes a Multi-view Visual Token Compressor (\textbf{MVTC}) based on multi-head attention mechanism, which adaptively compresses image tokens from both textual and visual views, thereby effectively reducing redundancy while retaining necessary information and avoiding modality conflicts. Additionally, we propose an attention pruning strategy to remove redundant attention layers in MLLMs, substantially reducing inference cost. To mitigate the performance degradation introduced by pruning, we further incorporate a linear projection for error compensation. Extensive experiments on four benchmark datasets demonstrate that ELMM achieves state-of-the-art performance.
\end{abstract}

\section{Introduction}
Knowledge Graphs (KGs) represent real-world knowledge in a structured form using factual triples (head, relation, tail). In recent years, they have attracted considerable attention from both academia and industry, and have been widely applied to various downstream tasks such as information retrieval \cite{2020Biomedical,wang2024knowledge}, and recommendation systems \cite{guo2020survey,jiang2024diffkg}. Multimodal Knowledge Graphs (MKGs) enhance traditional KGs by integrating information from multiple modalities, such as text and images, thereby enabling richer and more accurate knowledge representations. This multimodal fusion has shown to improve the performance of a wide range of intelligent systems \cite{guo2024lgmrec}. However, existing MKGs often suffer from incomplete information, which limits their effectiveness in practical applications. To address this challenge, the task of Multimodal Knowledge Graph Completion (MKGC) has emerged and has seen rapid development \cite{gao2025mixed,MyGO,MKGformer,LAFA}. The core objective of MKGC is to leverage multimodal data to enhance the expressiveness of entity representations and uncover latent knowledge, thereby completing missing information in MKGs.

Large models, through extensive pre- and post-training on massive datasets, exhibit emergent capabilities and encode rich real-world knowledge \cite{guo2025deepseek,xiong2025uniattn}, making them promising for KGC. For example, KICGPT \cite{Kicgpt} addresses the long-tail issue by re-ranking candidate triples from traditional methods, while MKGL \cite{mkgl} investigates whether LLMs can comprehend the structured representation of knowledge in the form of triples (head, relation, tail). However, current research largely overlooks multimodal inputs, leaving the potential of multimodal large language models (MLLMs) for MKGC underexplored.

In MKGs, entities are typically associated with various modalities of information, such as images and textual descriptions. For example, in datasets like FB15k-237-IMG \cite{FB15K-237} and WN18-IMG \cite{WN18}, each entity is linked to 10 images along with corresponding textual descriptions, providing rich multimodal contextual information to complement the structured triples. State-of-the-art MLLMs (e.g., LLaVA-1.5\cite{llava1.5}, Qwen2.5-VL \cite{qwen2.5vl}, and MobileVLM \cite{chu2023mobilevlm}) typically employ visual encoders, such as ViT \cite{vit}, to encode different regions of an image into a sequence of image tokens. These image tokens are then concatenated with text tokens and jointly processed by MLLMs in a unified framework.
\begin{figure}[t]
\centering
\includegraphics[width=0.9\columnwidth]{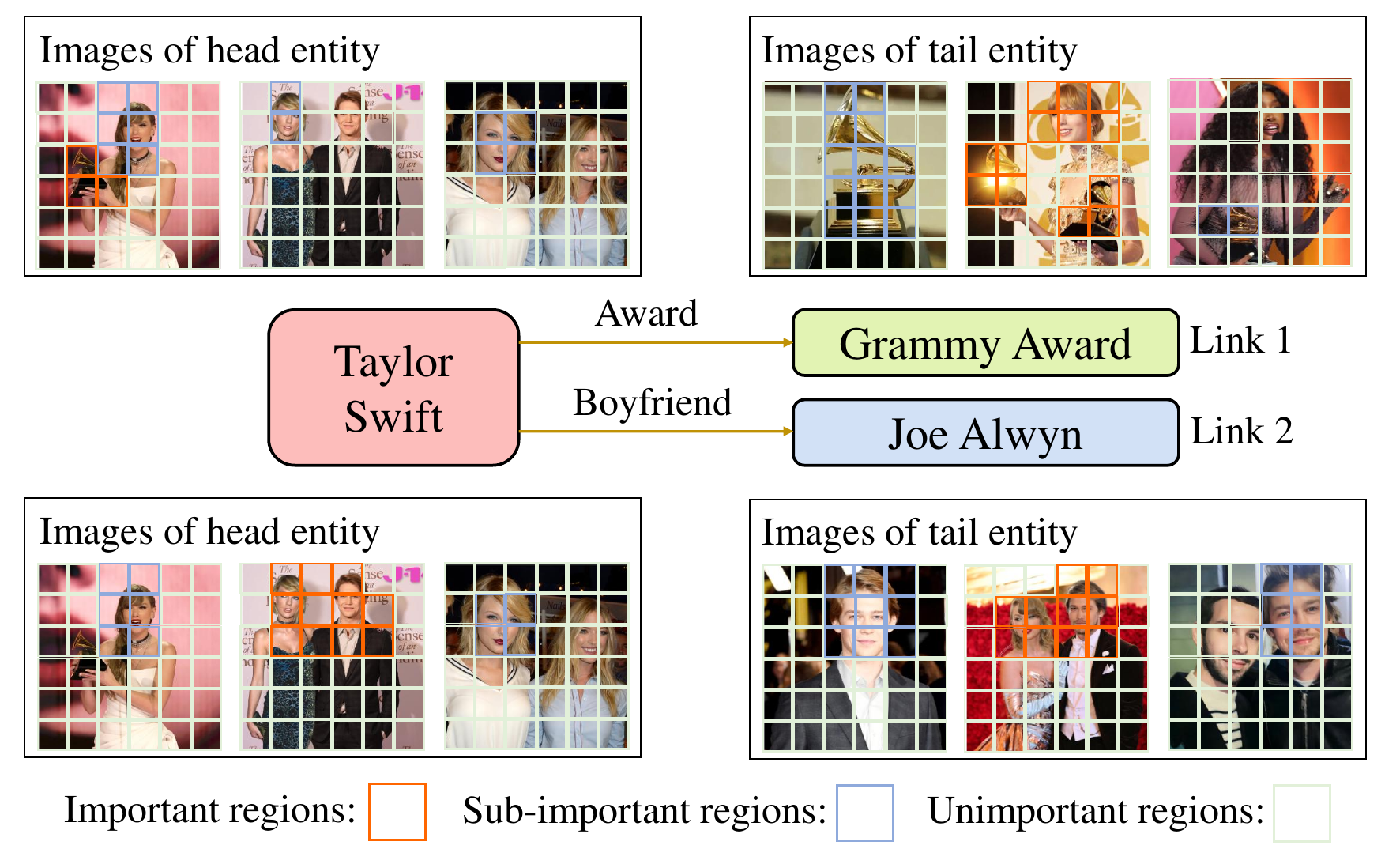}
\caption{A description of each associated image region's contribution to the entity Taylor Swift in a different link.}
\label{fig1}
\end{figure}
However, directly applying existing multimodal fusion paradigms to MKGC poses several challenges. First, \textbf{the number of image tokens per entity is large and often redundant.} Standard visual encoders generate hundreds of tokens per image; with 10 images per entity, this results in thousands of tokens. Moreover, the semantic relevance of visual regions is relation-dependent: as illustrated in Figure~\ref{fig1}, different relations emphasize different regions of Taylor Swift. Naively concatenating all visual and textual tokens therefore introduces noise and cross-modal interference, impeding effective semantic alignment. Second, \textbf{the computational cost of processing such large token inputs is prohibitive}. In MLLMs, the computational complexity of attention mechanisms typically grows quadratically with the number of tokens, while memory consumption increases at least linearly. This limits the model’s applicability in large-scale MKGC scenarios.

To address the above challenges, we propose ELMM, the first novel method to leverage \textbf{E}fficient \textbf{L}ightweight \textbf{M}ultimodal Large Language \textbf{M}odels (\textbf{ELMM}) for MKGC. First, we propose a \textbf{M}ulti-view \textbf{V}isual \textbf{T}oken \textbf{C}ompressor (\textbf{MVTC}) based on multi-head attention to select the most informative visual tokens from both textual and visual views. MVTC adaptively compresses visual tokens according to their semantic relevance to the textual description (entities and relations) and their intrinsic visual salience, aggregating tokens from multiple images into a fixed set of high-information representations. This design effectively reduces visual noise and modality interference, improving both inference efficiency and representational capacity. Second, we empirically identify structural redundancy in MLLM attention layers for MKGC, where certain layers produce highly similar input–output representations and contribute marginally to performance. To address this, we propose an attention pruning strategy that removes such redundant layers, significantly reducing inference overhead. We further observe that the pruning-induced error can be effectively compensated by a linear projection, and accordingly design a principled initialization pipeline for this projection.

Our contributions can be summarised as follows:
\begin{itemize}
    \item We propose \textbf{E}fficient \textbf{L}ightweight \textbf{M}ultimodal Large Language \textbf{M}odels (\textbf{ELMM}) for multimodal knowledge graph completion, a novel MKGC method that designs a \textbf{M}ulti-view \textbf{V}isual \textbf{T}oken \textbf{C}ompressor (\textbf{MVTC}) based on multi-head attention mechanism to adaptively compress image tokens from both textual and visual views. This mechanism enables effective compression of multimodal inputs while mitigating modality conflicts and reducing noise. To the best of our knowledge, this is the first work to address the MKGC task using MLLMs.
    \item We propose an attention pruning strategy that identifies and removes redundant attention layers in MLLMs for the MKGC task, thereby reducing computational overhead. To mitigate pruning-induced errors, we apply linear projection as compensation.
    \item Extensive experiments on four benchmark datasets demonstrate the superiority of ELMM in the MKGC task.
\end{itemize}
\begin{figure*}[t]
\centering
\includegraphics[width=0.92\textwidth]{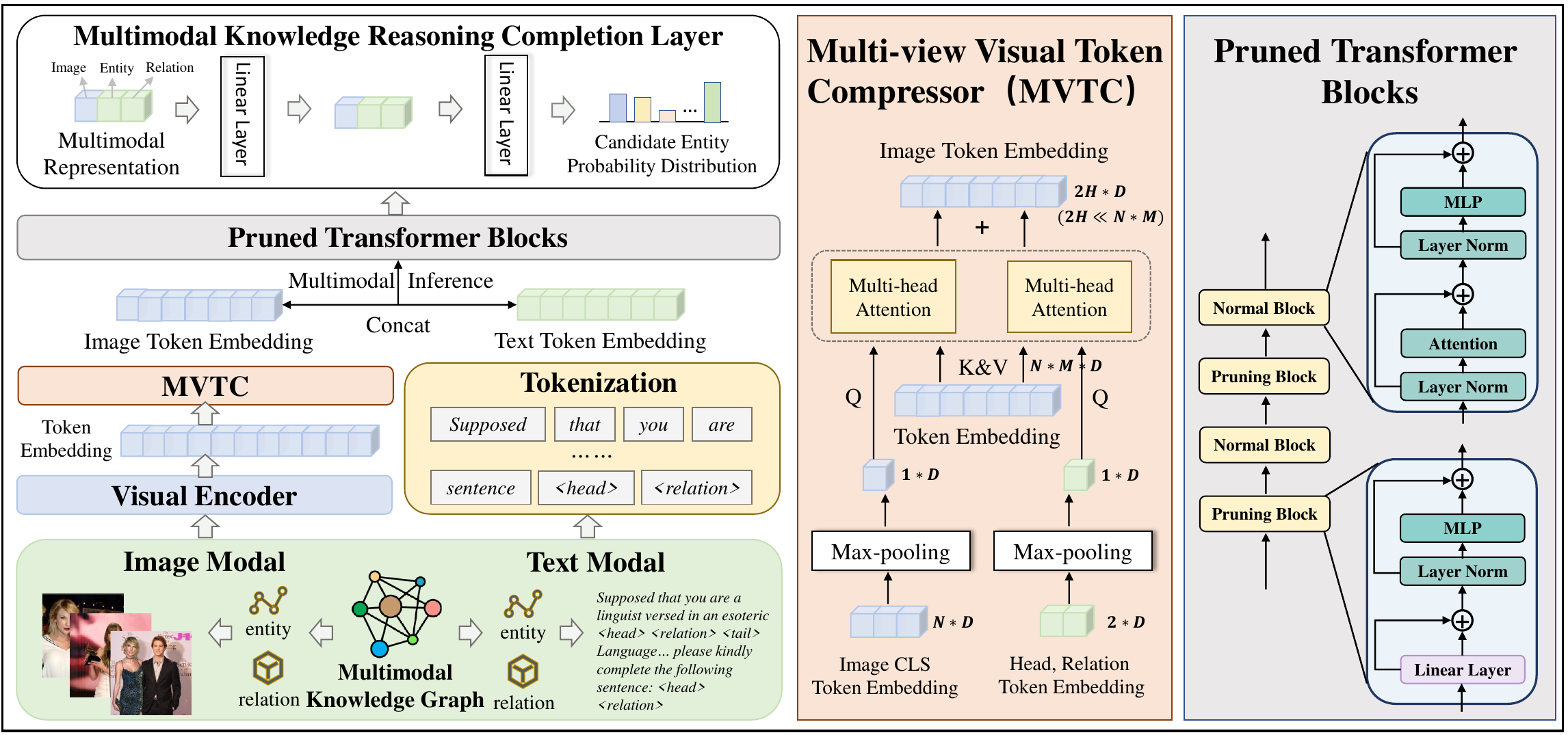} 
\caption{Overview of the ELMM architecture. From left to right: (1) ELMM processes input text and images containing incomplete triplet information; (2) MVTC adaptively compresses image tokens from both textual and visual views; (3) MLLM attention pruning and compensation mechanism.}
\label{fig2}
\end{figure*}

\section{Methodology}
\subsection{Preliminaries}
\subsubsection{Problem Formulation}
A multimodal knowledge graph (MKG) is defined as a directed graph $M K G=(\mathcal{E}, \mathcal{R}, \mathcal{G}, \mathcal{A}_{\mathcal{M}})$, where $\mathcal{E}$ and $\mathcal{R}$ denote the sets of entities and relations, respectively. The factual triples are represented as $\mathcal{G}=\{(e_{h}, r_{h, t}, e_{t}) \mid e_{h}, e_{t} \in \mathcal{E}, r_{h, t} \in \mathcal{R} \}$. Each entity is associated with a set of multimodal attributes $\mathcal{A}_{\mathcal{M}}$, which includes textual $\mathcal{M}_t$ and visual $\mathcal{M}_v$ modalities. The task of multimodal knowledge graph completion (MKGC) aims to infer missing entities in $\mathcal{G}$. Given an incomplete triple $(h, r, ?)$ or $(?, r, t)$, the goal is to predict the missing entity $t$ or $h$.
\subsubsection{Multimodal Large Language Models}
Generally, a multimodal large language model comprises the following components: a visual encoder and a textual tokenizer, which respectively segment images into multiple regions and encode it into a sequence of image token embeddings $I_{temp}\in \mathbb{R}^{NM\times E_i}$($N$ represents the number of images, $M$ represents the region of each image, and $E_i$ represents the dimension), while simultaneously converting the input text into a series of text token embeddings $T_t\in \mathbb{R}^{K\times D}$($K$ is the number of text tokens and $D$ is the dimension). The image token embeddings are then projected into the same dimensional space as the text embeddings via a linear transformation, yielding $I_t \in \mathbb{R}^{NM \times D}$. These are concatenated with the text embeddings to form a unified multimodal token sequence $T \in \mathbb{R}^{(NM + K) \times D}$; A series of transformer blocks $\mathcal{M}$, which processes the multimodal token sequence and generates a corresponding sequence of hidden states for prediction;
\begin{equation}
\mathbf{h}_{0: NM +  K-1}^{m}=\mathcal{M}\left(t_{0: NM + K -1}\right)
\end{equation}
and a head layer, which maps each hidden state to a probability distribution $\mathbf{p}_{n}=head(\mathbf{h}_{NM +  K}^{m})$.
\subsection{ELMM Framework}

The overall architecture of ELMM is depicted in Figure~\ref{fig2}. Applying conventional multimodal large language models (MLLMs) directly to MKGC generates excessive image tokens (e.g., over 1,000 tokens from 10 entity-related images after encoding), introducing noise and exacerbating modality conflicts. To mitigate these issues, ELMM proposes a \textbf{M}ulti-view \textbf{V}isual \textbf{T}oken \textbf{C}ompressor (\textbf{MVTC}) based on multi-head attention mechanism, which leverages multi-head attention to adaptively compress image tokens from both textual and visual views. MVTC effectively retains text-relevant information and key image details, thereby facilitating robust cross-modal alignment. Moreover, during MKGC, certain attention layers within the transformer blocks $\mathcal{M}$ exhibit high input-output similarity, indicating computational redundancy. ELMM addresses this by incorporating an \textbf{attention pruning strategy} that eliminates such redundant layers. A linear compensation module is then employed to mitigate approximation errors introduced by pruning, resulting in the pruned transformer blocks denoted as $\mathcal{M}'$. Finally, to better support MKGC, ELMM replaces the traditional head layer with a multimodal knowledge reasoning completion layer, improving cross-modal fusion and generating a probability distribution over candidate entities.
\subsection{Multi-view Visual Token Compressor}

The overall workflow of \textbf{M}ulti-view \textbf{V}isual \textbf{T}oken \textbf{C}ompressor (\textbf{MVTC}) is illustrated in the middle section of Figure \ref{fig2}. Considering that the importance of an image token can vary across different relational contexts (see Figure~\ref{fig1}), MVTC first starts from the textual view and compresses image tokens using entity and relation tokens from text. This process retains the token representations that are most strongly associated with text information, thereby achieving modality alignment. Specifically, we concatenate the entity and relation token embeddings, denoted as $T_{\mathrm{entity}}$ and $T_{\mathrm{relation}}$, respectively, and apply max-pooling to obtain a fused textual representation $X_t \in  \mathbb{R}^{1\times D}$. This representation $X_t$, together with the image token embeddings  $I_t$, is used as input to a Multi-Head Attention (MHA) mechanism. The attention computation is defined as:
\begin{equation}
Q_i=X_tW_i^Q, K_i=I_tW_i^K,Q_i\in \mathbb{R}^{1\times d},K_i\in \mathbb{R}^{MN\times d}
\end{equation}
\begin{equation}
V=I_tW^V,\quad V\in \mathbb{R}^{NM\times D}
\end{equation}
\begin{equation}
\begin{split}
Head_i &= Attention(Q_i, K_i, V) \\
       &= softmax\bigg(\frac{Q_i K_i^T}{\sqrt{d}}\bigg)V,\quad Head_i \in \mathbb{R}^{1 \times D}
\end{split}
\end{equation}
\begin{equation}
\begin{split}
I_{text}=Concat(Head_1,...,Head_H),I_{text} \in \mathbb{R}^{H\times D}
\end{split}
\end{equation}
Here, $W_i^Q, W_i^K, W^V$ are trainable projection matrices; $d$ denotes the dimensionality of the queries and keys, where $d = D / H$, and $H$ is the number of attention heads. The obtained $I_{text}$ represents the result of compressing the image tokens from the textual view. Meanwhile, to ensure the effective preservation of core information in images, MVTC also performs compression on image tokens from the visual view. Specifically, to preserve the most representative global information from the images, we first extract the CLS tokens from multiple images and concatenate them to form:
\begin{equation}
T_{\mathrm{CLS}}=\operatorname{concat}\left(\mathrm{CLS}_{1}, \mathrm{CLS}_{2}, \ldots, \mathrm{CLS}_{N}\right).
\end{equation}
We then apply a max-pooling operation over $T_{\mathrm{CLS}}$ to obtain a fused representation, denoted as $X_i$. Similar to the calculation process of $I_{text}$, we use $X_i$ and $I_T$ as the input of MHA to finally obtain $I_{image}$. The obtained $I_{image}$ contains the most important information of the image itself, representing the visual token information compressed from the visual view. We concatenate $I_{image}$, $I_{text}$, and $T_t$ to obtain $T$, and then reason through the pruned transformer blocks $\mathcal{M}'$ to obtain hidden states $ \mathbf{h}_{0: 2H + K-1}^{m}$. As a result, the aggregated hidden states retain not only essential visual information but also visual cues semantically aligned with the textual modality.

\subsection{Attention Pruning Strategy}
\begin{figure}[t]
\centering
\includegraphics[width=0.98\columnwidth]{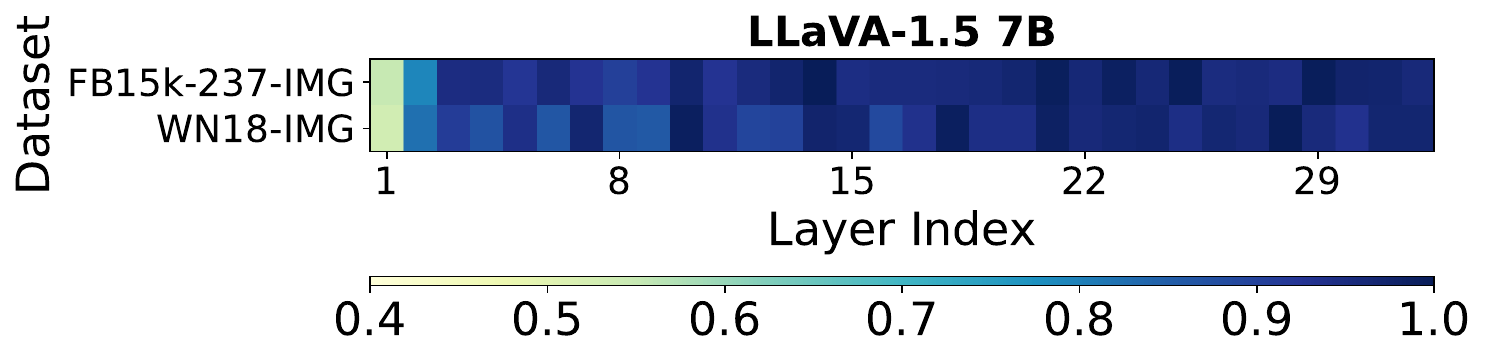}
\caption{Cosine similarity between the input and output of the attention layer in the training data of FB15k-237-IMG and WN18-IMG. The calculation method is to take 1,000 data points from the training set, perform inference, and calculate the average cosine similarity.}
\label{SimA}
\end{figure}
Due to the high computational cost of attention mechanisms, they often introduce significant inference latency. For instance, with an input sequence length of 512 tokens, the attention module alone can introduce up to 0.3437 seconds of latency per inference, accounting for approximately 77.9\% of the total inference time (0.4410 seconds). To mitigate this overhead, we perform a redundancy analysis of the attention modules in MLLMs, such as LLaVA-1.5 \cite{llava1.5}, during MKGC task. Specifically, we measure the cosine similarity between the input and output of each attention module to assess its redundancy. As illustrated in Figure \ref{SimA}, the upper layers of the model consistently exhibit high input-output similarity across various experimental settings. This observation suggests a significant degree of computational redundancy in certain attention modules of MLLMs during MKGC task.

Based on this finding, we propose an Attention Pruning Strategy that selects the top-$K$ layers with the highest input-output cosine similarity and prunes their attention modules. However, the direct removal of attention mechanisms inevitably introduces a forward propagation error, denoted as $\varepsilon = x^{ori} - x^{pruning}$, where $x^{ori}$ is the original output and $x^{pruning}$ is the output after pruning. To compensate for the error we utilize a linear projection $x^{pruning}W_c$ to fit the error $\varepsilon$. Then, we propose a training-free initialization method for $W_c$. Furthermore, we provide the following theorem to guide the initialization of $W_c$:
\begin{theorem}
\label{theorem:init}
    The initialization of $W_c$ that minimizes the expected compensation error satisfies:
    \begin{equation}
        W_{c}=V\Sigma^{+}U^T\mathbb{E}(\epsilon),
    \end{equation}
    where $U\Sigma V^T$ is the SVD decomposition of $\mathbb{E}(x^{pruning})$, $\Sigma^{+}$ denotes the pseudoinverse of $\Sigma$.
\end{theorem}
\begin{proof}
Minimizing the expected compensation error corresponds to solving the following optimization problem:
\begin{equation}
    \min_{W_c} \| \mathbb{E}[x^{pruning} W_c - \varepsilon] \|^2.
\end{equation}
Since $W_c$ is independent of the data distribution, the expectation operator can be applied to $x^{pruning}$ and $\varepsilon$ separately. Thus, the problem reduces to:
\begin{equation}
    \min_{W_c} \left\| \mathbb{E}[x^{pruning}] W_c - \mathbb{E}[\varepsilon] \right\|^2.
\end{equation}
This is a standard least-squares problem with the solution:
\begin{equation}
    W_c = \mathbb{E}[x^{pruning}]^+ \, \mathbb{E}[\varepsilon],
\end{equation}
where $\mathbb{E}[x^{pruning}]^+$ is the Moore–Penrose pseudoinverse of $\mathbb{E}[x^{pruning}]$. Given the SVD $\mathbb{E}[x^{pruning}] = U \Sigma V^\top$, the pseudoinverse is $\mathbb{E}[x^{pruning}]^+ = V \Sigma^+ U^\top$, leading to the final result:
\begin{equation}
    W_c = V \Sigma^+ U^\top \, \mathbb{E}[\varepsilon].
\end{equation}
\end{proof}
In practical applications, we randomly sampled 1,000 instances from the training dataset and employed the average values of $\bar{x}^{pruning}$ and $\bar{\epsilon}$ as estimators for $\mathbb{E}(x^{pruning})$ and $\mathbb{E}(\epsilon)$, respectively. Considering that errors in the early layers of the compensation model may propagate and affect the representations in subsequent layers, we adopted a bottom-up strategy to initialize $W_c$. Experimental results on the FB15k-237-IMG dataset demonstrate that, after applying linear compensation, the error of LLaVA-1.5 7B is significantly reduced from 276.5 (without compensation) to 4.7, corresponding to only 1.7\% of the original error. This significant improvement underscores the effectiveness of our linear compensation method.

\subsection{Multimodal Knowledge Reasoning Completion Layer}
To fit the multimodal knowledge graph completion task, we employ a multimodal knowledge reasoning completion layer to compute the probability distribution over candidate entities. The process is defined as follows:
\begin{equation}
    E_m=[E_{image},E_{entity},E_{relation}], S \in \mathbb{R}^{1\times 3D}
\end{equation}
\begin{equation}
    E_m'=LinearLayer(E_m),E_m' \in \mathbb{R}^{1\times 3D}
\end{equation}
\begin{equation}
    p=LinearLayer(E_m'),p \in \mathbb{R}^{1\times r}
\end{equation}
Here, $E_{image}$ is obtained by max-pooling over the hidden states of image modality hidden states. $E_{entity}$ and $E_{relation}$ correspond to the hidden states of the entity and relation tokens. These are concatenated to form the multimodal representation $E_{m}$. The fused representation is first processed by a linear layer and then projected to obtain the probability distribution $p \in \mathbb{R}^{1\times r}$ of the candidate entities, where $\mathit{r} $ represents the number of entities.

\subsection{Optimization}
To optimize ELMM, we employ a contrastive learning objective, defined as follows:
\begin{equation}
\begin{split}
\mathcal{L} &= \sum_{(e_i, r_k, e_j) \in \mathcal{T}_{train}} \bigg[ -\log (p_{e_{j}}) \\
            &\quad + \frac{1}{|\mathcal{N}_{neg}(e_{j})|} \sum_{e_{neg} \in \mathcal{N}_{neg}(e_{j})} \log (1-p_{e_{neg}}) \bigg].
\end{split}
\end{equation}
\begin{table}[t]
\centering
\resizebox{0.47\textwidth}{!}{
\begin{tabular}{cccccc}
    \midrule
    \textbf{Dataset} & \textbf{\#Rel.} & \textbf{\#Ent.} & \textbf{\#Train} & \textbf{\#Dev} & \textbf{\#Test} \\
    \midrule
    FB15k-237-IMG      & 237        & 14,541           & 272,115 
    & 17,535          &20,466
    \\    
    WN18-IMG      & 18        & 40,943           & 141,442
    & 5,000         &5,000
    \\
    DB15K      & 279        & 12842           & 79222
    & 9902         & 9904
    \\
    MKG-W      & 169        & 15000           & 34196
    & 4276   & 4274
    \\
    \midrule
\end{tabular}
}
\caption{Statistics of Dataset.}
\label{tab:my-data}
\end{table}
Here, $\mathcal{N}_{neg}\left(e_{j}\right)=\left\{e_{neg }\cup e_{i}  \mid e_{neg} \neq e_{j}, e_{neg} \in \mathcal{E}\right\}$ denotes the set of negative samples for the target entity $e_j$. Since the prediction of tail entities in the form of $(e_i, r_k, ?)$ often results in a bias where the head entity $e_i$ is assigned disproportionately high confidence, we incorporate a self-denoising strategy by treating the head entity $e_i$ as a hard negative sample. The loss function $\mathcal{L}$ is in a form of a binary cross-entropy.

\section{Experiments}
\begin{table*}[ht]
    \centering
    \resizebox{1 \textwidth}{!}{
    \begin{tabular}{cccccccccc}
\hline
\multirow{2}{*}{Model Name}  & \multicolumn{4
    }{c}{FB15k-237-IMG} &  & \multicolumn{4}{c}{WN18-IMG}    \\ \cline{2-5} \cline{7-10} 
                                                   & MR   & Hits@1  & Hits@3 & Hits@10 &  & MR  & Hits@1 & Hits@3 & Hits@10 \\   \hline 
\multicolumn{10}{c}{\textbf{Multi-modal KGC Models}} \\ \hline
TransAE(IJCNN 2019)                                   & 431  & 19.9   & 31.7  & 46.3   &  & 352 & 32.3  & 83.5  & 93.4   \\
RSME(ACMMM 2021)                                     & 417  & 24.2   & 34.4  & 46.7   &  & 223 & 94.3  & 95.1  & 95.7   \\
KG-BERT(2019)                  & 153  & -       & -      & 42.0   &  & 58  & 11.7  & 68.9  & 92.6   \\
VisualBERT(2019)            & 592  & 21.7   & 32.4  & 43.9   &  & 122 & 17.9  & 43.7  & 65.4   \\
VILBERT(ICLR 2020)             & 483  & 23.3   & 33.5  & 45.7   &  & 131 & 22.3  & 55.2  & 76.1   \\
MKGformer(SIGIR 2022)          & 221  & 25.6   & 36.7  & 50.4   &  & 28  & 94.4  & 96.1  & 97.2   \\
LAFA(AAAI 2024)       & \underline{136}  & 26.9   & 39.8  & 55.1   &  & \underline{25}  & 94.7  & 96.5  & 97.7   \\ 
NativE(SIGIR 2024)   & 149  & 25.4   & 38.6  & 54.2   &  & 35  & 94.2  & 95.8  & 97.1   \\ 
SGMPT$^{*}$(ACMMM 2024)        & 238  & 25.2   & 37.0  & 51.0   &  & 29  & 94.3  & \underline{96.6}  & 97.8   \\
MyGO$^{*}$(AAAI 2025)         & -    & 19.0    & 28.9  & 44.7   &  & –   & 70.6  & 93.7  & 94.1   \\
MPIKGC$^{*}$(COLING 2024)  & -    & 24.4   & 35.8  & 50.3   &  & –   & -      & -      & -       \\
AdaMF-MAT$^{*}$(COLING 2024)      & –    & 23.1   & 35.0  & 49.1  &  & –   & 73.6  & 94.3  & 95.8   \\\hline \multicolumn{10}{c}{\textbf{Uni-modal KGC Models}}    \\ \hline
KICGPT (EMNLP 2023)      & 154    & 32.7   & 44.8  & 55.4  &  &  \underline{25}  & 93.5  & 95.6  & 96.8   \\ 
MKGL (NeurIPS 2024)      & 172   & 32.5   & 45.4  & 59.1  &  & 31   & 93.6  & 95.4  & 96.3   \\ 
K-ON (AAAI 2025)  & 144   & 33.2   & 45.7  & 59.8  &  & 26   & 94.3  & 95.7  & 97.2   \\ 
GLTW (ACL 2025)  & -   & \underline{35.1}   & \underline{48.1}  & \underline{61.4}  &  & \underline{17}   & \underline{95.2}  & 96.3  & \underline{97.9}   \\
PEKGC (EMNLP 2025)  & -   & 33.6   & 46.6  & 54.4  &  & 30   & 94.8  & 96.0  & 96.9   \\
\hline
ELMM(Ours)                  & \textbf{105}    & \textbf{37.4}  & \textbf{50.2} & \textbf{63.7} &  & \textbf{11}  & \textbf{96.1} & \textbf{97.8} & \textbf{98.9}  \\ \hline
\footnotesize{${*}$ indicates that the result comes from SGMPT \cite{SGMPT}.}
\end{tabular}
}
\caption{Models performance on the FB15k-237-IMG and WN18-IMG datasets. \textbf{Bold} and \underline{underline} text indicate the best and second-best performance, respectively.}
\label{tab:my-table}
\end{table*}
\subsection{Datasets}
We evaluate the performance of ELMM on four publicly available multimodal knowledge graph datasets: FB15k-237-IMG \cite{FB15K-237}, WN18-IMG \cite{WN18}, DB15K\cite{DB15K}  and MKG-W\cite{MMRNS}. Each dataset comprises three modalities: (1) structured knowledge graph triples, (2) textual descriptions of entities, and (3) a set of associated images per entity. Detailed statistics for both datasets are presented in Table \ref{tab:my-data}. Due to space constraints, we report evaluation results primarily on the FB15k-237-IMG and WN18-IMG datasets in the main text; \textbf{results on DB15K and MKG-W are provided in Appendix \ref{sec:appendix_ER}}.

\subsection{Setting}
We adopt LLaVA-1.5 7B \cite{llava1.5} as the base multimodal large language model (MLLM), incorporating Low-Rank Adaptation (LoRA) \cite{hu2022lora} techniques into both the query and value layers to enable parameter-efficient fine-tuning. The parameter $K$ in the attention pruning strategy is always set to 16, indicating that 16 attention modules in the MLLM are pruned and replaced with linear projections. Full hyperparameter configurations are provided in the Appendix \ref{sec:appendix_ED}.

To evaluate the effectiveness of ELMM on the Multimodal Knowledge Graph Completion (MKGC), we adopt standard evaluation metrics, including Hits@k (including Hits@1, Hits@3, and Hits@10) to assess ranking precision at different thresholds, and Mean Rank (MR) to measure the average ranking position of the correct target entities among the predictions.

For a comprehensive performance comparison on FB15k-237-IMG and WN18-IMG datasets, we benchmark our model against two representative categories of methods: (i) multimodal knowledge graph completion methods, and (ii) uni-modal knowledge graph completion methods. The former category includes TransAE \cite{TransAE}, RSME \cite{RSME}, KG-BERT \cite{kg-bert}, VisualBERT \cite{visualbert}, ViLBERT \cite{Vilbert}, MKGformer \cite{MKGformer}, LAFA \cite{LAFA}, NativE \cite{zhang2024native}, SGMPT \cite{SGMPT}, MyGO \cite{MyGO}, MPIKGC \cite{MPIKGC}, and AdaMF-MAT \cite{AdaMF-MAT}. The LLM-based uni-modal knowledge graph completion baselines include KICGPT \cite{Kicgpt}, K-ON \cite{kno}, GLTW \cite{GLTW}, PEKGC \cite{SAT} and MKGL \cite{mkgl}.


\subsection{Comparison with State-of-the-art}

We conduct a comprehensive comparison between the proposed ELMM method and 18 state-of-the-art models across FB15k-237-IMG and WN18-IMG datasets. As shown in Table \ref{tab:my-table}, ELMM consistently outperforms all competing methods across all evaluation metrics. Notably, on the FB15k-237-IMG dataset, ELMM demonstrates substantial improvements over most powerful multimodal approaches, achieving relative gains of 39.0\%, 26.1\%, and 15.6\% in Hits@1, Hits@3, and Hits@10, respectively. These improvements primarily stem from ELMM’s ability to directly leverage the expressive reasoning capabilities of MLLMs to jointly model multimodal information within a unified semantic space. As a result, ELMM more effectively resolves ambiguous cases and entities that require fine-grained visual reasoning. These results strongly demonstrate the potential of MLLMs for MKGC. Furthermore, compared to uni-modal methods based on large language models, ELMM exhibits clear advantages on both datasets, underscoring its effectiveness and robustness in extracting and integrating multimodal information.
\subsection{Ablation Study}
To evaluate the effectiveness of each component in ELMM, we conduct a series of ablation studies by removing or modifying key modules of the model. Specifically: w/o Image removes the visual views image token representations ($I_{image}$); w/o Text removes the textual view image token representations ($I_{text}$); w/o MVTC removes MVTC and does not compress image tokens; w/o pruning indicates the removal of the attention pruning strategy, and MLLM is not pruned; w/o Linear removes the linear compensation mechanism and replaces it with a standard residual connection; w/o Init replaces the initialization strategy of the compensation matrix $W_c$ proposed in Theorem 1 with zero initialization; and head layer replaces the multimodal knowledge reasoning completion layer with a conventional head-layer.

As shown in Table \ref{compare two commongeneration losses}, w/o Text performance has declined significantly, while w/o Image has not declined significantly. This indicates that image tokens associated with textual context are more critical in the MKGC task. One possible explanation is that the importance of the same token may vary across different relations (as illustrated in Figure \ref{fig1}), a unique aspect of MKGC that $I_{image}$ representations fail to fully capture. Moreover, removing the MVTC component leads to a substantial performance drop in ELMM, primarily due to increased noise and intensified modality conflicts caused by the large number of image tokens. This ablation result further confirms that directly applying MLLMs to the MKGC task yields suboptimal outcomes. \textbf{It highlights the necessity of designing targeted image token compression and selection mechanisms to mitigate cross-modal information interference specific to this task.}

Notably, ELMM achieves over 30\% faster inference without significant performance degradation  compared to the non-pruned model, indicating redundancy in attention modules for the MKGC task. Moreover, removing the linear compensation or initializing its projection weights $W_c$ to zero both lead to comparable performance drops, suggesting that the effectiveness of linear compensation depends critically on appropriate initialization. Also, removing the multimodal knowledge reasoning completion layer similarly hurts overall performance.

\begin{table}[t]
\centering
\resizebox{0.47\textwidth}{!}{
\begin{tabular}{lccccc}
\hline
Model               &   MR  & Hits@1 & Hits@3 & Hits@10  \\ \hline
ELMM        & \underline{105} & \textbf{37.4}  & \textbf{50.2}  & \underline{63.7}  \\ 
- w/o Image   & 115 & 37.6  & 49.8  & 63.1        \\ 
- w/o Text   & 146 & 33.5  & 46.3  & 60.7        \\
- w/o MVTC   & 187 & 28.5  & 41.5  & 56.2        \\
- w/o Pruning & \textbf{102} & 37.1  & \underline{50.0}  & \textbf{64.2}   \\
- w/o Linear      & 146 & 35.9  & 47.7  & 60.6      \\
- w/o Init     & 140 & 36.8  & 47.5  & 61.5      \\
- Head Layer     & 119 & \underline{37.2}  & 49.2  & 63.0       \\  
\hline
\end{tabular}
}
\caption{Ablation study results for the FB15k-237-IMG dataset. Bold and underline denote the best and second-best
performance of compressed models.}
\label{compare two commongeneration losses}
\end{table}

\subsection{Parameter Sensitivity Analysis}
\begin{figure}[t]
\centering
\includegraphics[width=0.99\columnwidth]{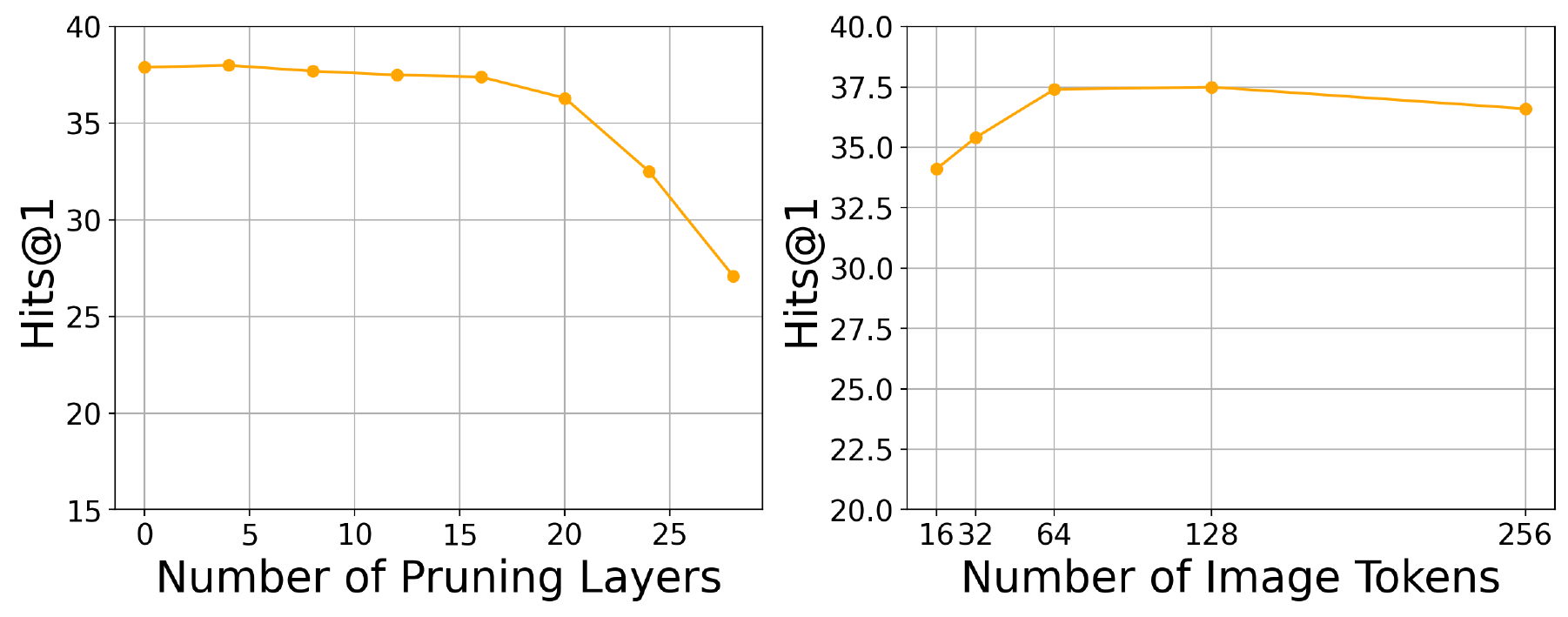}
\caption{Hyperparameter analysis results. Left: Relationship between Hits@1 and the number of attention pruning layers. Right: Relationship between Hits@1 and number of image tokens are retained after MVTC.}
\label{ana}
\end{figure}
\subsubsection{Number of Pruning Layers}
In our propose attention pruning strategy, we employ a similarity-based criterion to identify and prune the top-$K$ most similar layers. We evaluate the impact of pruning layers on performance, and the corresponding results are shown in Figure \ref{ana}(left). Experimental findings indicate that the model maintains stable performance when the number of pruned layers is relatively small, implying a degree of redundancy in the attention modules of multimodal large language models (MLLMs) for MKGC. However, a marked degradation in performance is observed when the pruning depth exceeds 17 layers. Based on this analysis, we fix the pruning depth at 16 layers to achieve a favorable trade-off between inference efficiency and task accuracy.
\subsubsection{Number of Image Tokens are Retained After MVTC}
We systematically evaluate the impact of the number of retained image tokens on model performance, as shown in Figure \ref{ana}(right). Overall, performance improves as more image tokens are preserved, indicating the positive contribution of visual information to multimodal knowledge graph completion. However, when the number of retained image tokens exceeds 128, performance degrades, suggesting that excessive visual tokens may introduce redundancy or noise that hampers effective modeling. Balancing predictive performance and computational efficiency, we retain 64 visual tokens in subsequent experiments.
\subsection{Impact of Backbone MLLMs on MKGC Performance}
To investigate whether ELMM can continuously benefit from the rapid evolution of MLLMs, we study the impact of different backbone models on completion performance. Specifically, we select LLaVA-1.5 7B \cite{llava1.5}, LLaVA-1.5 13B \cite{llava1.5}, and Qwen2.5-VL 7B \cite{qwen2.5vl} as comparative models. The results are presented in Table \ref{base_model}.

From our experimental results, we observe that while LLaVA-1.5 13B contains a greater number of parameters than LLaVA-1.5 7B, the corresponding performance gains are relatively limited. This outcome may be attributed to the modest difference in parameter count between the two models. In contrast, improvements in the backbone architecture appear to have a more substantial impact on performance. For example, Qwen2.5-VL 7B consistently outperforms LLaVA-1.5 7B and even LLaVA-1.5 13B on several evaluation metrics.
\begin{table}[t]
\centering
\resizebox{0.47\textwidth}{!}{
\begin{tabular}{cccccc}
\hline
Base Model               &   MR  & Hits@1 & Hits@3 & Hits@10  \\ \hline
LLaVA-1.5 7B        & 105 & 37.4  & 50.2  & 63.7  \\ 
LLaVA-1.5 13B   & 102 & 37.8  & 50.7  & 64.0        \\ 
Qwen2.5-VL 7B      & 83 & 38.6  & 52.2  & 64.9      \\
\hline
\end{tabular}
}
\caption{Results of ELMM using different base models on the FB15k-237-IMG dataset.}
\label{base_model}
\end{table}
\begin{table}[t]
\centering
\begin{tabular}{cccccc}
\hline
Model Name & Inference Time &   MR  & Hits@1   \\ \hline
ELMM &  0.081   & 105 & 37.4    \\ 
MKGL  &   0.126  & 172 & 32.5    \\ 
K-ON  &   0.084  & 144 & 33.2 \\ 
\hline
\end{tabular}
\caption{Comparison of ELMM with state-of-the-art uni-modal methods based on large language models in terms of inference time and model performance.}
\label{Infer}
\end{table}
\subsection{Inference Time Analysis}

We conduct experiments on the FB15k-237-IMG dataset to compare the end-to-end reasoning latency of ELMM with that of state-of-the-art unimodal large language model–based methods, including MKGL and K-ON. Specifically, we randomly sample 1,000 instances from the training set for inference, and report the average latency as the inference time. As shown in Table~\ref{Infer}, despite the additional cost of processing multimodal inputs, ELMM exhibits lower inference latency than unimodal baselines. This efficiency gain is primarily attributed to the proposed attention pruning strategy, which effectively reduces redundant computation during reasoning. These results demonstrate that ELMM not only outperforms unimodal methods in reasoning efficiency, but also achieves strong performance on standard evaluation metrics by effectively modeling and integrating multimodal information.
\section{Conclusion}

We propose Efficient Lightweight Multimodal Large Language Models (ELMM), a novel framework that extends MLLMs to the task of MKGC. ELMM addresses two key challenges: (1) semantic noise and modality conflicts caused by excessive image tokens, and (2) the substantial computational overhead associated with processing such inputs using standard MLLMs. To mitigate these issues, we propose a \textbf{M}ulti-view \textbf{V}isual \textbf{T}oken \textbf{C}ompressor (\textbf{MVTC}) based on multi-head attention mechanism, which adaptively compresses image tokens from both textual and visual views. We further apply attention pruning to remove redundant attention layers from the MLLMs backbone, improving inference efficiency. To offset any performance degradation from pruning, we add a linear projection with a carefully designed initialization scheme. Experiments on  four benchmark datasets, show that ELMM achieves state-of-the-art performance.

\section*{Limitations}

Although ELMM demonstrates promising effectiveness, it relies on powerful pretrained multimodal backbone models, such as LLaVA and Qwen2.5-VL. While our framework is largely model-agnostic, its performance and efficiency gains are still constrained by the representational capacity of the underlying multimodal large language models (MLLMs).

\bibliography{custom}

\appendix
\section{License for Scientific Artifacts}

The Training dataset FB15k-237-IMG \cite{FB15K-237} is licensed under Creative Commons Attribution 4.0 License\footnote{https://choosealicense.com/licenses/cc-by-4.0/}. The Training dataset WN18-IMG \cite{WN18} is licensed under WordNet Release 3.0 License\footnote{https://wordnetcode.princeton.edu/3.0/README}. The Training dataset DB15K \cite{DB15K} is licensed under CC BY-SA 3.0 License\footnote{https://creativecommons.org/licenses/by-sa/3.0/deed.en}. The Training dataset MKG-W \cite{MMRNS} is licensed under CC BY-SA 3.0 License\footnote{https://creativecommons.org/licenses/by-sa/3.0/deed.en}. The LLaVA-1.5 model \cite{llava1.5} and Qwen2.5-VL model \cite{qwen2.5vl} are licensed under Apache License 2.0\footnote{https://choosealicense.com/licenses/apache-2.0/}. All usages of scientific artifacts in this paper obey the corresponding licenses.
\section{Related Works}
\label{sec:appendix_RW}

\subsection{Multimodal Knowledge Graph Completion}
Multimodal knowledge graph completion (MKGC) seeks to enhance the reasoning ability of knowledge graphs by jointly leveraging structural triples and multimodal entity information such as images and text. Early approaches \cite{TransAE,RSME,IKRL} typically used simple feature concatenation, but recent research has introduced more advanced techniques for fine-grained and adaptive multimodal fusion. Representative models include MKGformer \cite{MKGformer}, which employs hybrid transformers with multi-level fusion; SGMPT \cite{SGMPT} introduces structural information based on MKGformer; LAFA \cite{LAFA}, which adopts link-aware fusion and neighbor aggregation; MyGO \cite{MyGO}, which tokenizes multimodal entity information for fine-grained representation;  and NativE \cite{zhang2024native}, which addresses modality imbalance in real-world scenarios.

With the rapid advancement of multimodal large language models (MLLMs), these models have demonstrated substantial potential in task related to Multimodal Knowledge Graphs (MKGs). However, most studies have yet to fully explore and exploit the capabilities of MLLMs. To address this gap, we propose ELMM, a novel method designed to harness the strengths of MLLMs.

\begin{table}[t]
    \centering
    \resizebox{0.48 \textwidth}{!}{
    \begin{tabular}{ccccc}
\hline
Datasets & MLLM   & LoRA r  & LoRA dropout & LoRA Target  \\   \hline 
FB15k-237-IMG  & LLaVA-1.5 7B  & 32   & 0.05  & query, value    \\
WN18-IMG     & LLaVA-1.5 7B  & 32   & 0.05  & query, value   \\  
DB15K  & LLaVA-1.5 7B  & 32   & 0.06  & query, value   \\  
MKG-W     & LLaVA-1.5 7B  & 32   & 0.06  & query, value   \\  
\hline 

\end{tabular}
}
\caption{LoRA settings in the main experiments.}
\label{tab1}
\end{table}
\begin{table}[t]
\centering
\resizebox{0.47\textwidth}{!}{
\begin{tabular}{cccccc}
    \midrule
    Datasets & Layer of Pruning   \\   \hline 
    FB15k-237-IMG  & 13,24,28,20,22,29,12,30,9,19,23,21,18,31,5,26     \\
    WN18-IMG     & 27,17,9,20,13,23,30,6,31,14,25,22,21,26,28,19    \\
    DB15K & 15,27,10,19,17,13,20,29,21,22,31,30,28,23,25,14    \\
    MKG-W & 19,8,17,16,27,30,23,11,22,29,20,24,21,26,18,31    \\
    \midrule
\end{tabular}
}
\caption{Attention pruning settings in the main experiments.}
\label{tab2}
\end{table}
\subsection{Multimodal Large Language Models}

With the rapid development of Multimodal Large Language Models (MLLMs), such as GPT-4 \cite{achiam2023gpt}, LLaVA-1.5 \cite{llava1.5}, Qwen2.5-VL \cite{qwen2.5vl}, and Emu2 \cite{Emu2}, vision and language understanding have been deeply integrated within a unified semantic space. These models typically adopt a paradigm where visual encoders \cite{vit} segment images into multiple regional patches, which are then encoded into dense image token sequences. These image tokens are projected to align with the dimensionality of text tokens and concatenated with textual inputs for joint processing by transformer backbones, enabling strong performance in visual question answering, image captioning, and general knowledge reasoning tasks.

In recent years, the redundancy of visual tokens has been widely recognized as a major source of increased computational and memory costs in multimodal models, prompting growing interest in addressing the inefficiency caused by excessive image tokens \cite{Deco,IPCV,wen2025stop,yang2025efficientvlatrainingfreeaccelerationcompression}. For instance, Qwen2.5-VL \cite{qwen2.5vl} incorporates dynamic resolution adjustment, DeCo \cite{Deco} performs parameter-free adaptive pooling compression at the patch level, and IPCV \cite{IPCV} mitigates visual token redundancy through shallow-layer token pruning and neighbor-guided reconstruction for recovering pruned tokens. 

However, existing methods compress visual tokens solely based on visual view, which is not well suited for the MKGC task. As demonstrated in Figure \ref{fig1} and further validated by our ablation studies, \textbf{visual tokens that are semantically aligned with the textual modality play a more critical role in MKGC, a factor largely overlooked by prior approaches}. To address this limitation, we propose MVTC, which leverages textual information to identify and preserve the most relevant visual tokens while compressing those that are less informative. This design facilitates more effective multimodal alignment and fusion in subsequent reasoning stages. Furthermore, to enable the application of MLLMs to MKGC, ELMM proposes several additional innovations beyond MVTC, including an attention pruning strategy and a multimodal knowledge reasoning completion layer. Together, these components allow ELMM to achieve state-of-the-art performance on MKGC benchmarks.

\section{Further Experiment Details}
\label{sec:appendix_ED}
\subsection{Training Hyperparameters.}
We adopt LLaVA-1.5 7B \cite{llava1.5} as the base Multimodal Large Language Model (MLLM), incorporating Low-Rank Adaptation (LoRA) \cite{hu2022lora} techniques into both the query and value layers to enable parameter-efficient fine-tuning. For LoRA related parameter settings, please refer to Table \ref{tab1}. All experiments are conducted using eight NVIDIA A100 GPUs, and both training and inference are implemented with the PyTorch framework. The parameter $K$ in the attention pruning strategy is always set to 16, indicating that 16 attention modules in the MLLM are pruned and replaced with linear projections. To determine which layers to prune, we randomly sample 1,000 training instances and compute the cosine similarity between the attention maps of adjacent layers for each dataset. Layers exhibiting the highest similarity—indicating strong redundancy—are selected for pruning. The specific pruned layers for each dataset are reported in Table~\ref{tab2}. The key parameter $H$ in MVTC represents the number of attention headers, and $H$ is set to 32 in all datasets, which represents 32 image tokens each for the textual and visual views, resulting in a total of 64 image tokens. For training parameter settings, please refer to Table \ref{tab12}. We report the version numbers of used packages in Table \ref{tab:package_version}.

\begin{table}[b]
\centering
\resizebox{0.47\textwidth}{!}{
\begin{tabular}{cccccc}
    \midrule
    Datasets & MLLM  & Batch Size  & Optimizer &  Epoch &  Learning Rate \\   \hline 
    FB15k-237-IMG  & LLaVA-1.5 7B  & 32 & Adam  & 5  & 3e-4   \\
    WN18-IMG     & LLaVA-1.5 7B  & 16 & Adam  & 2  & 2e-4   \\
    DB15K     & LLaVA-1.5 7B  & 64 & Adam  & 2  & 1e-4   \\
    MKG-W      & LLaVA-1.5 7B  & 64 & Adam  & 2  & 1e-4   \\
    \midrule
\end{tabular}
}
\caption{Training settings in the main experiments.}
\label{tab12}
\end{table}

\subsection{Prompt used in the experiment.}
Across all datasets, we employ a unified training and evaluation prompt. Specifically, for instances where the task is to predict the tail entity given a \textbf{head} entity and a \textbf{relation}, we use the following prompt formulation:
\begin{verbatim}
    "Next you will be given the Head entity, 
    which has image representations: 
    <image>.
    
    Suppose that you are an excellent linguist 
    studying a three-word language. 
    Given the following dictionary:
    
    Input\tType\tDescription
    {h}\tHead entity\t{h_des}
    {r}\tRelation\t{r_des}
    
    Please complete the last word (?) 
    of the sentence: {h}{r}?"
\end{verbatim}
Similarly, for instances where the task is to predict the head entity given a \textbf{relation} and a \textbf{tail} entity, we use the following prompt formulation:
\begin{verbatim}
    "Next you will be given the Head entity, 
    which has image representations: 
    <image>.
    
    Suppose that you are an excellent linguist 
    studying a three-word language. 
    Given the following dictionary:
    
    Input\tType\tDescription
    {t}\tHead entity\t{t_des}
    {inv_r}\tRelation\t{inv_r_des}
    
    Please complete the last word (?) 
    of the sentence: {t}{inv_r}?"
\end{verbatim}
\begin{table}[!h]
\centering
\scalebox{0.82}{\begin{tabular}
{lc|lc}
\toprule
 Package & Version & Package & Version \\
\midrule
PyTorch & 2.0.0 & transformers & 4.48.0 \\
deepspeed & 0.10.0 & tokenizers & 0.20.1 \\
scatter & 2.1.2 & sparse & 0.6.18 \\
datasets & 2.14.3 &  &  \\
\bottomrule
\end{tabular}}
\caption{Versions of used packages.}
\label{tab:package_version}
\end{table}
\begin{table*}[ht]
    \centering
    \resizebox{1 \textwidth}{!}{
    \begin{tabular}{cccccccccc}
\hline
\multirow{2}{*}{Model Name}  & \multicolumn{4
    }{c}{DB15K} &  & \multicolumn{4}{c}{MKG-W}    \\ \cline{2-5} \cline{7-10} 
                                                   & MRR   & Hits@1  & Hits@3 & Hits@10 &  & MRR  & Hits@1 & Hits@3 & Hits@10 \\   \hline 
\multicolumn{10}{c}{\textbf{Baseline Models}} \\ \hline
IKRL(IJCAI 2017)      & 26.8 & 14.1 & 34.9 & 49.1 && 32.4 & 26.1 & 34.8 & 44.1   \\
TBKGC(NAACL 2018)     & 28.4 & 15.6 & 37.0 & 49.9 && 31.5 & 25.3 & 34.0 & 43.2   \\
TransAE(IJCNN 2019)   & 28.1 & 21.3 & 31.2 & 41.2 && 30.0 & 21.2 & 34.9 & 44.7   \\
MMKRL(APIN 2025)   & 26.8 & 13.9 & 35.1 & 49.4 && 30.1 & 22.2 & 34.1 & 44.7   \\
RSME(MM 2021)   & 29.8 & 24.2 & 32.1 & 40.3 && 29.2 & 23.4 & 32.0 & 40.4  \\
VBKGC(KDD 2022)  & 30.6 & 19.8 & 37.2 & 49.4 && 30.6 & 24.9 & 33.0 & 40.9\\
OTKGE(NeurIPS 2022)  & 23.9 & 18.5 & 25.9 & 34.2 && 34.4 & 28.9 & 36.3 & 44.9 \\
IMF(WWW 2023)     & 32.3 & 24.2 & 36.0 & 48.2 && 34.5 & 28.8 & 36.6 & 45.4   \\ 
QEB(MM 2023)  & 28.2 & 14.8 & 36.7 & 51.6 && 32.4 & 25.5 & 35.1 & 45.3   \\ 
VISTA(EMNLP 2023)   & 30.4 & 22.5 & 33.6 & 45.9 && 32.9 & 26.1 & 35.4 & 45.6\\
AdaMF(COLING 2024)   & 32.5 & 21.3 & 39.7 & 51.7 && 34.3 & 27.2 & 37.9 & 47.2   \\
MyGO(AAAI 2025)  & 37.7 & 30.1 & 41.3 & 52.2 && 36.1 & 29.8 &38.5 & 47.8 \\
MCKGC(AAAI 2025)      & \underline{39.8}    & \underline{31.9}   & \underline{43.8}  & \underline{54.7}  && \underline{36.9} &  \underline{31.3}  &  \underline{38.9}  & 47.4   \\
K-ON (AAAI 2025)  & 38.1   & 30.1   & 42.8  & 53.6  &  & 36.6   &  30.1  & 38.7  & \underline{48.3}   \\ 
\hline
ELMM(Ours)   & \textbf{41.2}    & \textbf{34.1}  & \textbf{45.5} & \textbf{56.7} &  & \textbf{38.4}  & \textbf{33.5} & \textbf{41.7} & \textbf{51.5}  \\ \hline
\end{tabular}
}
\caption{Models performance on the DB15K and MKG-W datasets. \textbf{Bold} and \underline{underline} text indicate the best and second-best performance, respectively.}
\label{tab:my-table2}
\end{table*}
\section{More Experiment Results.}
\label{sec:appendix_ER}
\subsection{Main Results on DB15K and MKG-W.}
\subsubsection{Comparison Methods and Evaluation Protocol.}
Since many methods evaluated on FB15k-237-IMG and WN18-IMG are not tested on DB15K and MKG-W, and vice versa, we re-select the comparison methods to enable a fair and meaningful evaluation on the DB15K and MKG-W datasets. The selected baselines primarily include IKRL~\cite{IKRL}, TBKGC~\cite{TBKGC}, TransAE~\cite{TransAE}, MMKRL~\cite{MMKRL}, RSME~\cite{RSME}, VBKGC~\cite{vbkgc}, OTKGE~\cite{OTKGE}, IMF~\cite{IMF}, QEB~\cite{QEB}, VISTA~\cite{VISTA}, AdaMF~\cite{AdaMF-MAT},MCKGC \cite{MCKGC}, MyGO~\cite{MyGO},and K-ON \cite{kno}.

We adopt standard evaluation metrics, including Hits@k (Hits@1, Hits@3, and Hits@10) to measure ranking accuracy at different cutoffs, and Mean Reciprocal Rank (MRR) to evaluate the overall ranking performance of each model.

\subsubsection{Main results.}
We conduct a comprehensive comparison between the proposed ELMM and 14 state-of-the-art methods on the DB15K and MKG-W datasets. As reported in Table~\ref{tab:my-table2}, ELMM consistently outperforms all competing approaches across all evaluation metrics, which is in line with the results observed on the FB15k-237-IMG and WN18-IMG datasets. These results demonstrate the robustness and general effectiveness of ELMM for multimodal knowledge graph completion, and further highlight the potential of multimodal large language models (MLLMs) in advancing multimodal knowledge graph completion.

\subsection{Robustness Experimental Test.}
We conducted a systematic evaluation of the robustness of ELMM under modality-missing scenarios. Specifically, to simulate modality incompleteness in real-world applications, we randomly removed either textual descriptions or visual information from a subset of training samples, while keeping the evaluation procedures unchanged. This setting is designed to assess the ability of ELMM to operate effectively under incomplete multimodal inputs.We considered modality-missing rates of 20\%, 40\%, 60\%, and 80\%. For example, a 20\% missing rate indicates that 10\% of the training samples lack textual information and an additional 10\% lack visual information; higher missing rates follow the same proportion. 

All experiments were conducted on the FB15k-237-IMG dataset, with Hits@1 adopted as the evaluation metric. The experimental results are illustrated in Figure \ref{Robustness}. An interesting phenomenon can be observed: when the modality missing rate is set to 20\%, the model achieves a noticeable performance improvement compared to the fully observed setting. We hypothesize that a moderate degree of modality missing increases the training difficulty, thereby exerting a regularization effect that encourages the model to learn more robust representations.

As the modality missing rate further increases, the overall Hits@1 performance exhibits a downward trend. Nevertheless, even under relatively high missing rates, the performance degradation remains limited. This observation indicates that ELMM demonstrates strong robustness in modality-missing scenarios. We attribute this robustness to the ability of multimodal large language models to leverage their internal knowledge and contextual reasoning capabilities to partially compensate for missing modality information, thereby mitigating the adverse effects caused by modality incompleteness.
\begin{figure}[t]
\centering
\includegraphics[width=0.99\columnwidth]{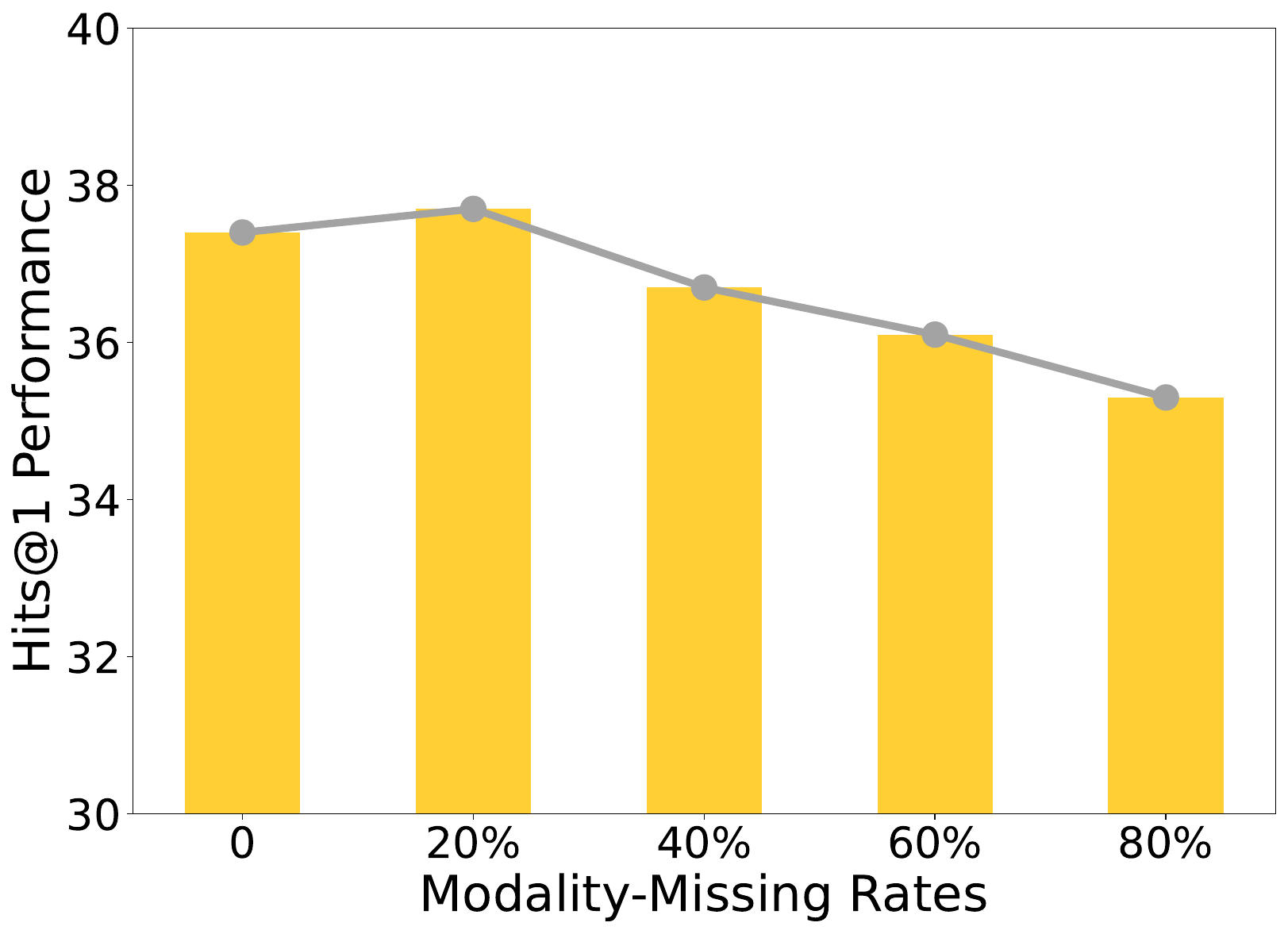}
\caption{Performance under varying modality-missing rates on FB15k-237-IMG. The y-axis corresponds to Hits@1, and the x-axis represents the proportion of missing modality information.}
\label{Robustness}
\end{figure}

\end{document}